\documentclass{llncs}

\usepackage{wrapfig,appendix}
\usepackage{graphicx}
\usepackage{epsfig}
\usepackage{verbatim}

\usepackage{algorithm,algorithmic}
\usepackage{url,setspace, floatflt}

\begin{document}
\title{The Role of Weight Shrinking in Large Margin Perceptron Learning}
\author{Constantinos Panagiotakopoulos \and Petroula Tsampouka}
\institute{Physics Division, School of Technology \\ Aristotle University of Thessaloniki, Greece \\
\email{costapan@eng.auth.gr, petroula@gen.auth.gr} }
\maketitle
\begin{abstract}
We introduce into the classical perceptron algorithm with margin a mechanism that shrinks the current weight vector as a first step of the update. If the shrinking factor is constant the resulting algorithm may be regarded as a margin-error-driven version of NORMA with constant learning rate. In this case we show that the allowed strength of shrinking depends on the value of the maximum margin. We also consider variable shrinking factors for which there is no such dependence. In both cases we obtain new generalizations of the perceptron with margin able to provably attain in a finite number of steps any desirable approximation of the maximal margin hyperplane. The new approximate maximum margin classifiers appear experimentally to be very competitive in 2-norm soft margin tasks involving linear kernels.
%\\

%{\bf Keywords:} Online learning, classification, maximum margin.
\end{abstract}
\renewcommand{\vec}[1]{\mbox{\boldmath$#1$}}
\newcommand{\Tiny}[1]{\mbox{\tiny$#1$}}

\section{Introduction}
It is widely accepted that the generalization ability of learning machines improves as the margin of the solution hyperplane increases \cite{Vap}. The simplest online learning algorithm for binary linear classification, the perceptron \cite{Ros,Nov}, does not aim at any margin. The problem, instead, of finding the optimal separating hyperplane is central to Support Vector Machines (SVMs) \cite{Vap,CST}.

SVMs obtain large margin solutions by solving a constrained quadratic optimization problem using dual variables. In the early days, however, efficient implementation of SVMs was hindered by the quadratic dependence of their memory requirements on the number of training examples. To overcome this obstacle decomposition methods \cite{Plat,Joa} were developed that apply optimization only to a subset of the training set. Although such methods led to considerable improvement the problem of excessive runtimes when processing very large datasets remained. Only recently the so-called linear SVMs \cite{Joa06,HCL,PT2} by making partial use of primal notation in the case of linear kernels managed to successfully deal with massive datasets.

The drawbacks of the dual formulation motivated research long before the advent of linear SVMs in alternative large margin classifiers naturally formulated in primal space. Having the perceptron as a prototype they focus on the primal problem by updating a weight vector which represents their current state whenever a data point presented to them satisfies a specific condition. By exploiting their ability to process one example at a time\footnote{The conversion of online algorithms to the batch setting is done by cycling repeatedly through the dataset and using the last hypothesis for prediction.} they save time and memory and acquire the potential to handle large datasets. The first algorithm of the kind is the perceptron with margin \cite{DH} the solutions of which provably possess only up to $1/2$ of the maximum margin \cite{KM}. Subsequently, various others succeeded in approximately attaining maximum margin by employing modified perceptron-like update rules. For ROMMA \cite{LL} such a rule is the result of a relaxed optimization which reduces all constraints to just two. In contrast, ALMA \cite{Gen} and much later CRAMMA \cite{TST1} and MICRA \cite{TST2} employ a ``projection" mechanism to restrict the length of the weight vector and adopt a learning rate and margin threshold in the condition which both follow specific rules involving the number of updates. Very recently, the margitron \cite{PT1} and the perceptron with dynamic margin (PDM) \cite{PT3} using modified conditions managed to approximately reach maximum margin solutions while maintaining the original perceptron update.
 
A somewhat different approach from the hard margin one adopted by most of the algorithms above was also developed which focuses on the minimization of the regularized 1-norm soft margin loss through stochastic gradient descent. Notable representatives of this approach are the pioneer NORMA \cite{KSW} and Pegasos \cite{SSS}. Stochastic gradient descent gives rise to perceptron-like updates  an important ingredient of which is the ``shrinking" of the current weight vector. Shrinking is always imposed when a pattern is presented to the algorithm with it being the only modification suffered by the weight vector in the event that its condition is violated and as a consequence no loss is incurred. The cummulative effect of shrinking is to gradually diminish the impact of the earlier contributions to the weight vector. Shrinking has also been employed by algorithms which do not have their origin in stochastic gradient descent as an accompanying mechanism in perceptron-based budget scenarios for classification \cite{DSS} or tracking \cite{CBG}.

Our purpose in the present work is to investigate the role that shrinking of the weight vector might play in large margin perceptron learning. This is motivated by the observation that such a mechanism naturally emerges in attempts to attack the 1-norm soft margin task through stochastic gradient descent. If we accept that algorithms like NORMA succeed in minimizing the regularized 1-norm soft margin loss they should be able to solve the hard margin problem as well for sufficiently small non-zero values of the regularization parameter which also controls the strength of shrinking. Thus shrinking, as weak as it may be, when introduced into the perceptron algorithm with margin might prove beneficial. Another factor to be taken into account is that the shrinking mechanism in the algorithms considered here is operative only for erroneous trials, a feature that offers them the possibility to terminate in a finite number of steps. Therefore, shrinking in such algorithms may need to be strengthened relative to algorithms like NORMA to compensate for the fact that the latter shrink the weight vector even when the condition is violated. In conclusion, the amount of shrinking that a perceptron with margin could tolerate without it destroying the conservativeness of the update might be sufficient to raise the theoretically guaranteed fraction of the maximum margin achieved to a value larger than $1/2$. It turns out that this is actually the case.

The remaining of this paper is organized as follows. Section 2 contains some preliminaries and a description of the algorithms. In Sect. 3 we present a theoretical analysis of the algorithms. Section 4 is devoted to implementational issues and a brief experimental evaluation while Sect. 5 contains our conclusions.

\section{The Algorithms}
Let us consider a linearly separable training set $\{(\vec x_k, l_k)\}^m_{k=1}$, with vectors $\vec x_k\in \bbbr^d$ and labels $l_k \in \{+1,-1\}$. This training set may be either the original dataset or the result of a mapping into a feature space of higher dimensionality \cite{Vap,CST}. By placing $\vec x_k$ in the same position at a distance $\rho$ in an additional dimension, i.e., by extending $\vec x_k$ to $[\vec x_k, \rho]$, we construct an embedding of our data into the so-called augmented space \cite{DH}. This way, we construct hyperplanes possessing bias in the non-augmented feature space. Following the augmentation, a reflection with respect to the origin of the negatively labeled patterns is performed. This allows for a uniform treatment of both categories of patterns.  Also, $R\equiv\displaystyle \max_{k} \left\| \vec{y}_{k} \right\|$ with $\vec{y}_{k} \equiv [l_k\vec x_k, l_k\rho]$ the $k^{\rm {th}}$ augmented and reflected pattern.
 
The relation characterizing optimally correct classification of the training patterns $\vec{y}_{k}$ by a 
weight vector $\vec{u}$ of unit norm in the augmented space is
\begin{equation}
\label{gamma}
\vec{u} \cdot \vec{y}_{k}\ge \gamma_{\rm d}\equiv \displaystyle \max_{ \vec{u}^{\prime}:\left\|\vec{u}^{\prime}\right\|=1}
\displaystyle \min_{i}\left \{\vec{u}^{\prime} \cdot \vec{y}_{i}\right \}  \ \ \ \ \forall k \enspace.
\end{equation}
We shall refer to $\gamma_{\rm d}$ as the maximum directional margin. It coincides with
the maximum margin in the augmented space with respect to hyperplanes passing through
the origin. The maximum directional margin $\gamma_{\rm d}$ is upper bounded by the maximum geometric margin $\gamma$ in the non-augmented space and tends to it as $\rho \to \infty$ \cite{TST}.  

We consider algorithms in which the augmented weight vector $\vec{a}^s_{t}$ is initially set to zero, i.e. $\vec{a}^s_{0}=\vec0$, and is updated according to the perceptron-like rule
\begin{equation}
\label{update}
\vec{a}^s_{t+1}=c^s_t\vec{a}^s_{t}+\eta\vec{y}_k
\end{equation}
each time the ``misclassification" condition
\begin{equation}
\label{cond}
{\bar c}^s_t\vec a^s_t \cdot \vec y_k \le b
\end{equation}
is satisfied by a training pattern $\vec{y}_{k}$, i.e., whenever a margin error is made on $\vec{y}_{k}$. Here $0<c^s_t,{\bar c}^s_t \le 1$ are ``shrinking factors" which may vary with the number $t$ of updates, $\eta>0$ is a constant learning rate and $b>0$ acts as a margin threshold in the misclassification condition. If we set $c^s_t={\bar c}^s_t=1$ we recover the classical perceptron algorithm with margin. The role of $c^s_t$ is to shrink the current weight vector as a first step of the update, thereby enhancing the importance of the current update relative to the previous one. At the same time such a shrinking acts as a mechanism of effectively increasing the margin threshold of the condition, an effect that may be further strengthened through the introduction of the factor ${\bar c}^s_t$ in (\ref{cond}). In fact, for appropriate choices of $c^s_t,{\bar c}^s_t$, to which we confine our interest here, it is possible to equivalently introduce shrinking into the perceptron with margin via a learning rate and margin threshold which both increase with $t$. Notice that we denote by $\vec a^s_t$ the weight vector of the algorithms with shrinking to reserve the notation $\vec a_t$ for the weight vector of the equivalent algorithms with variable learning rate and margin threshold.

\begin{wrapfigure}{l}{0.55\textwidth}
\vspace{-0pt}
\epsfig{file=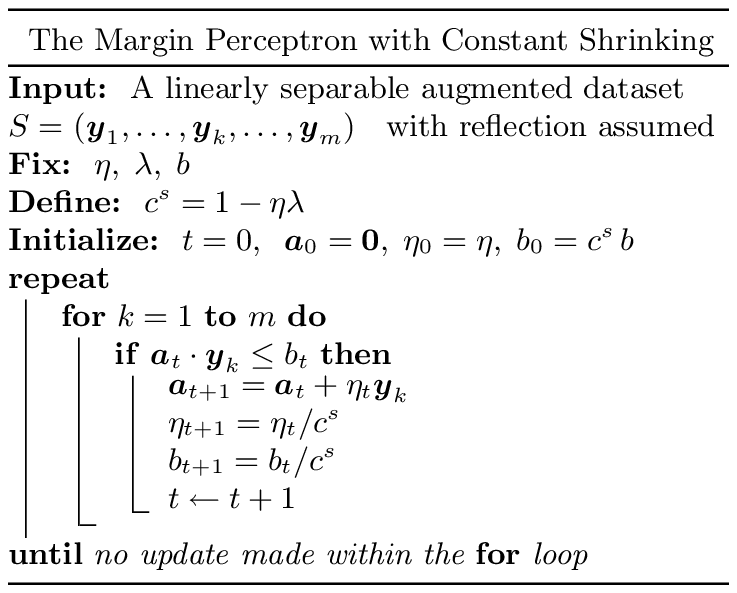, width=0.54\textwidth}
\vspace{-10pt}
\end{wrapfigure}
We investigate the impact of shrinking on large margin perceptron learning by considering both constant and variable shrinking factors. If shrinking does not depend on $t$ we set ${\bar c}^s_t=1$ since a constant ${\bar c}^s_t$ may be absorbed into a redefinition of $b$. We also express $c^s_t$ in terms of a ``shrinking parameter" $\lambda<1/\eta$ as $c^s_t=1-\eta \lambda$. Then (\ref{update}) becomes the update of NORMA for $\vec a^s_t \cdot \vec y_k \le b$. NORMA, however, updates its weight vector even when $\vec a^s_t \cdot \vec y_k > b$. In this case the update reduces to pure shrinking $\vec{a}^s_{t+1}=(1-\eta \lambda )\vec{a}^s_{t} $. This is the important difference from our algorithm in which an update occurs only if the misclassification condition is satisfied, thereby making convergence in a finite number of steps possible. 

Let us divide the update rule (\ref{update}) with $(1-\eta \lambda)^t$ and condition (\ref{cond}) with $(1-\eta \lambda)^{t-1}$. Also let $\vec a_t=\vec a^s_t/(1-\eta \lambda)^{t-1}$. Then, we obtain a completely equivalent algorithm with update
\begin{equation}
\label{update12}
\vec{a}_{t+1}=\vec{a}_{t}+\frac{\eta}{(1-\eta \lambda)^t} \vec{y}_k
\end{equation}
and condition
\begin{equation}
\label{cond12}
\vec a_t \cdot \vec y_k \le \frac{b}{(1-\eta \lambda)^{t-1}} \enspace.
\end{equation}

An algorithm with variable shrinking is obtained if we choose $c^s_t={\bar c}^s_t=\left(t/(t+1)\right)^n$, where $n \ge 0$ is an integer. For $n=1$ the shrinking factor $c^s_t$ entering the update is the one encountered in Pegasos. Pegasos, however, has variable learning rate, ${\bar c}^s_t=1$ and performs, just like NORMA, a pure shrinking update when its condition is violated. In addition, its update ends with a projection step. A variable shrinking factor $t/(t+\lambda)$ is also employed by SPA \cite{CBG} in which $b=0$. Such a factor is related to ours for $\lambda=n$ even if $n \neq 1$ since  for $t \gg n$
% $t/(t+\lambda) \approx 1/e^{\lambda/t}=\left(1/e^{1/t}\right)^\lambda\approx \left(t/(t+1)\right)^{\lambda}$ provided %$t\gg 1,\lambda$.
\[
\frac{t}{t+n}=\prod^{n-1}_{k=0}\frac{t+k}{t+k+1}=\left(\frac{t}{t+1}\right)^n \prod^{n-1}_{k=0}\left(1+\frac{k}{t(t+k+1)}\right)\approx\left(\frac{t}{t+1}\right)^n \enspace.
\]
%Actually, $t/(t+\lambda) \approx \left(t/(t+1)\right)^{\lambda}$ holds for non-integer $\lambda \ll t$ as well.

Let us multiply both the update rule (\ref{update}) and condition (\ref{cond}) with $(t+1)^n$ and set $\vec a_t=t^n \vec a^s_t$. Then, we obtain a completely equivalent algorithm with update
\begin{equation}
\label{update22}
\vec{a}_{t+1}=\vec{a}_{t}+\eta (t+1)^n \vec{y}_k
\end{equation}
and condition
\begin{equation}
\label{cond22}
\vec a_t \cdot \vec y_k \le b(t+1)^n \enspace.
\end{equation}
\begin{wrapfigure}{l}{0.55\textwidth}
\vspace{-10pt}
\epsfig{file=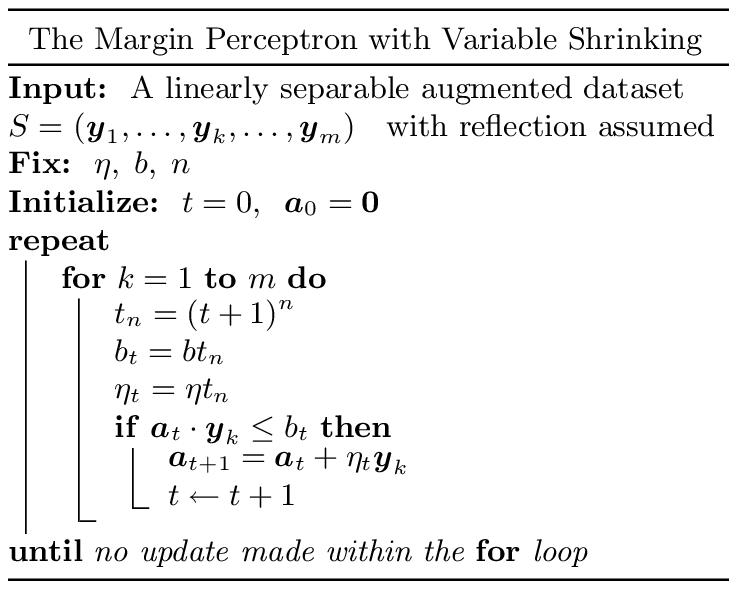, width=0.54\textwidth}
\vspace{-10pt}
\end{wrapfigure}
If we had chosen ${\bar c}^s_t=1$ we should have multiplied (\ref{cond}) with $t^n$. As a result the threshold in (\ref{cond22}) would have been $bt^n$, a difference that does not seem to be of paramount importance. However, the choice ${\bar c}^s_t=\left(t/(t+1)\right)^n$ prevailed for the sake of convenience. The choice, instead, $c^s_t={\bar c}^s_t=t/(t+n)$$=P^n_t/P^n_{t+1}$ with $P^n_t \equiv \prod^{n-1}_{k=0}(t+k)$ would have led to $\vec a_t=P^n_t \vec a^s_t$ and to the replacement of $(t+1)^n$ with $P^n_{t+1}$ in (\ref{update22}) and (\ref{cond22}).

We shall refer to the algorithm with update (\ref{update12}) and condition (\ref{cond12}) as the margin perceptron with constant shrinking. The algorithm, instead, with update (\ref{update22}) and condition (\ref{cond22}) will be called the margin perceptron with variable shrinking. The above formulations of the algorithms are the ones that will henceforth be considered in place of the original formulations of (\ref{update}) and (\ref{cond}).

\section{Theoretical Analysis}
We begin with the analysis of the margin perceptron with constant shrinking.

\begin{theorem}
\label{theorem1}
The margin perceptron with constant shrinking converges in a finite number $t_{\rm c}$ of updates satisfying the bound
\begin{equation}
\label{tbound1}
t_{\rm c} \le \frac{1}{\delta (1-\epsilon)} \frac{R^2}{\gamma^2_{\rm d}} \ln \frac{4-(2+\delta)\epsilon+\delta}{(2+\delta)\epsilon-\delta}
\end{equation}
provided $\delta \equiv {\eta R^2} /{b} \le 2$ and $\epsilon \equiv 1- \lambda b /\gamma^2_{\rm d}$ obey the constraint $\delta/(2+\delta)<\epsilon<1$.
Moreover, the zero-threshold solution hyperplane possesses margin $\gamma^{\prime}_{\rm d}$ which is a fraction $f$ of the maximum margin $\gamma_{\rm d}$ obeying the inequality 
\begin{equation}
\label{fbound1}
f \equiv \frac{\gamma^{\prime}_{\rm d}}{\gamma_{\rm d}} >\frac{1}{2+\delta}+\frac{1-\epsilon}{2}.
\end{equation}
Finally, an after-run lower bound on $f$ involving the margin $\gamma^{\prime}_{\rm d}$ achieved, the length $\left\| \vec a_{t_{\rm c}} \right\|$ of the solution weight vector $\vec a_{t_{\rm c}}$ and the number $t_{\rm c}$ of updates is given by
\begin{equation}
\label{fboundafter1}
f  \ge  \frac{1-(1-\eta \lambda)^{t_{\rm c}}}{\lambda (1-\eta \lambda)^{t_{\rm c}-1}} \frac {\gamma^{\prime}_{\rm d}}{\left\| \vec a_{t_{\rm c}} \right\|}.
\end{equation}
\end{theorem}

\begin{proof}
Taking the inner product of (\ref{update12}) with the optimal direction $\vec u$ and using (\ref{gamma}) we get
\[
\vec u \cdot \vec a_{t+1} -\vec u \cdot \vec a_{t} =\frac{\eta}{(1-\eta \lambda)^{t}} \vec u \cdot \vec y_k \ge \frac{ \eta}{(1-\eta \lambda)^{t}}\gamma_{\rm d}
\]
a repeated application of which, taking into account that $\vec a_0= \vec 0$, gives 
\begin{equation}
\label{lbound1}
\left\| \vec a_t \right\| \ge \vec u \cdot \vec a_t \ge \sum_{k=0}^{t-1}\frac{\eta \gamma_{\rm d}}{(1-\eta \lambda)^{k}} =\frac{1-(1-\eta \lambda)^{t}}{\lambda (1-\eta \lambda)^{t-1}} \gamma_{\rm d}\enspace.
\end{equation}
Here we made use of $\sum_{k=0}^{t-1} \alpha^k=(\alpha^t-1)/(\alpha-1)$. From (\ref{update12}) and (\ref{cond12}) we obtain
\[
\left\|\vec a_{t+1}\right\|^2 -\left\|\vec a_t\right\|^2=\frac{\eta^2}{(1-\eta \lambda)^{2t}} \left\|\vec y_k\right\|^2 +\frac{2\eta}{(1-\eta \lambda)^{t}} \vec a_t \cdot \vec y_k \le  
 \frac{\eta^2 R^2+2\eta(1-\eta \lambda)b }{(1-\eta \lambda)^{2t}}\enspace.
 \]
A repeated application of the above inequality, assuming $\vec a_0= \vec 0$, leads to
\begin{equation}
\label{upbound1}
\left\|\vec a_{t}\right\|^2 \le \sum_{k=0}^{t-1}\frac{\eta^2 R^2+2\eta(1-\eta \lambda) b } {(1-\eta \lambda)^{2k}}=\frac{\left(1-(1-\eta \lambda)^{2t}\right){\left(\eta R^2+2(1-\eta \lambda) b\right) }}{\lambda(2-\eta \lambda)(1-\eta \lambda)^{2(t-1)}} \enspace.
\end{equation}
Comparing the lower bound on $\left\|\vec a_{t}\right\|^2$ from (\ref{lbound1}) with its upper bound (\ref{upbound1}) we get
\begin{equation}
\label{tbound121}
\frac{\left(1-(1-\eta \lambda)^t\right)^2}{\lambda} \gamma^2_{\rm d} \le \frac{1-(1-\eta \lambda)^{2t}}{2-\eta \lambda}\left(\eta R^2+2(1-\eta \lambda) b\right)
\end{equation}
or, noticing that $1-(1-\eta \lambda)^{2t}=\left(1-(1-\eta \lambda)^t\right)\left(1+(1-\eta \lambda)^t\right)$, we obtain
\begin{equation}
\label{tbound122}
1-(1-\eta \lambda)^t \le \left(1+(1-\eta \lambda)^t\right) {\cal A} \enspace.
\end{equation}
Here
\begin{equation}
\label{tbound123}
{\cal A} \equiv \lambda \left(\frac{b}{\gamma^2_{\rm d}}\right) \frac{\eta R^2/b+2(1-\eta \lambda) }{2-\eta\lambda}<1 \enspace.
\end{equation}
The condition ${\cal A}<1$ ensures that (\ref{tbound122}) does lead to an upper bound on the number of updates since otherwise (\ref{tbound122}) is always satisfied. This translates into a very restrictive upper bound on the shrinking parameter $\lambda$. This upper bound depends on the values of the remaining parameters but is never larger than $\gamma^2_{\rm d}/b$. From (\ref{tbound122}), provided ${\cal A}<1$, we easily derive the following upper bound on the number of updates
\begin{equation}
\label{tbound124}
t\le t_{\rm b} \equiv \frac{1}{\ln (1-\eta\lambda)^{-1}} \ln{\frac{1+{\cal A}}{1-{\cal A}}} \enspace.
\end{equation}
For $\delta \le 2$ it holds that 
\begin{equation}
\label{inqq}
\frac{\delta+2(1-\eta\lambda)}{2-\eta\lambda}\le 1+\frac{\delta}{2}
\end{equation}
and
\begin{equation}
\label{tbound125}
{\cal A}=(1-\epsilon)\left(\frac{\delta+2(1-\eta\lambda)}{2-\eta\lambda} \right) \le (1-\epsilon)\left(1+\frac{\delta}{2}\right)=1-\frac{(2+\delta)\epsilon-\delta}{2} \enspace.
\end{equation}
As a consequence, $\epsilon > \delta/(2+\delta)$ ensures that ${\cal A}<1$. In addition
\begin{equation}
\label{tbound126}
\ln (1-\eta\lambda)^{-1} \ge \eta\lambda = \delta(1-\epsilon)\frac{\gamma^2_{\rm d}}{R^2} \enspace.
\end{equation}
Combining (\ref{tbound124}), (\ref{tbound125}) and (\ref{tbound126}) we finally arrive at the slightly simplified upper bound on the number of updates given by (\ref{tbound1}).

Upon convergence of the algorithm in $t_{\rm c}$ updates condition (\ref{cond12}) is violated by all patterns. Therefore, the achieved margin $\gamma^{\prime}_{\rm d} > {b}/\left({(1-\eta\lambda)^{t_{\rm c}-1}\left\| \vec a_{t_{\rm c}} \right\|}\right)$.
Thus, 
\begin{eqnarray}
f^2=\frac{{\gamma^{\prime}_{\rm d}}^2}{\gamma_{\rm d}^2}&>&\frac{b^2}{(1-\eta\lambda)^{2(t_{\rm c}-1)}\left\| \vec a_{t_{\rm c}} \right\|^2\gamma_{\rm d}^2} \ge \frac{\lambda (2-\eta \lambda)b^2}{\left(1-(1-\eta \lambda)^{2t_{\rm c}}\right)\left(\eta R^2+2(1-\eta \lambda) b\right)\gamma_{\rm d}^2 } \nonumber \\
&\ge&\frac{\lambda (2-\eta \lambda)b^2}{\left(1-(1-\eta \lambda)^{2t_{\rm b}}\right)\left(\eta R^2+2(1-\eta \lambda) b\right)\gamma_{\rm d}^2}=\left(\frac{\lambda b}{\left(1-(1-\eta \lambda)^{t_{\rm b}}\right)\gamma_{\rm d}^2}\right)^2, \nonumber
\end{eqnarray}
where use has been made of the upper bound (\ref{upbound1}) on $\left\| \vec a_{t_{\rm c}} \right\|^2$ and of the fact that (\ref{tbound121}) at $t=t_{\rm b}$ holds as an equality. Taking the square root and making use of the definition of $t_{\rm b}$ from (\ref{tbound124}) the previous inequality becomes
\[
f> \lambda \frac{b}{\gamma_{\rm d}^2}\left(\frac{1+{\cal A}}{2{\cal A}}\right)=\frac{1}{2}\left(\frac{\lambda b}{{\cal A}\gamma_{\rm d}^2}+\lambda \frac{b}{\gamma_{\rm d}^2} \right)= \frac{1}{2}\left(\frac{2-\eta\lambda}{\delta+2(1-\eta\lambda)} + 1-\epsilon \right)\enspace.
\]
For $\delta \le 2$ the above inequality gives rise to (\ref{fbound1}) because of (\ref{inqq}).

Finally, (\ref{fboundafter1}) is readily obtained if in the ratio ${\gamma^{\prime}_{\rm d}}/{\gamma_{\rm d}}$ we employ the upper bound on $\gamma_{\rm d}$ derivable from (\ref{lbound1}). \qed
\end{proof}

\begin{remark}
\label{rm1}
The parameters $\delta$ and $\epsilon$ are independent. Therefore, we may consider choosing $\delta \ll 1$ while keeping $\epsilon$ fixed. In this case the upper bound (\ref{tbound1}) on the number of updates becomes $O \left(\delta^{-1}R^2/\gamma_{\rm d}^2 \right)$ and from (\ref{fbound1}) the before-run lower bound on $f$ approaches as $\delta \to 0$ the value $1-\epsilon/2$. This generalizes the well-known result that the classical perceptron algorithm with margin (obtainable when $\lambda \to 0$ or $\epsilon \to 1$) has in the limit $\delta \to 0$ a theoretically guaranteed before-run value of $f$ equal to $1/2$. By subsequently letting $\epsilon \to 0$ (i.e., $\lambda \to \gamma^2_{\rm d}/b$) we may approach solutions with maximum margin.
\end{remark}

\begin{remark}
\label{rm2}
To facilitate comparison with other large margin classifiers we may relate the independent parameters $\delta$ and $\epsilon$ and obtain a single parameter $\zeta<1/\sqrt{2}$ through the relations $\delta=2\zeta$, $\epsilon=\delta(1+\delta)/(2+\delta)=\zeta(1+2\zeta)/(1+\zeta)$. Then, from (\ref{tbound1}) and (\ref{fbound1}) we have that the margin perceptron with constant shrinking achieves ``accuracy'' $\zeta$, i.e.,
\[ 
f>1-\zeta \enspace,
\] 
in a number $t_{\rm c}$ of updates satisfying the bound
\[
t_{\rm c} \le \frac{1}{\zeta}\left(\frac{1+\zeta}{1-2\zeta^2}\right) \frac{R^2}{\gamma^2_{\rm d}} \ln \frac{\sqrt{1-\zeta^2}}{\zeta} \enspace.
\]
Notice that the quantity ${R}/{\gamma_{\rm d}}$ does not enter the logarithm. In this sense the above bound, which is $O \left((\zeta^{-1}R^2/\gamma_{\rm d}^2)\ln \zeta^{-1}\right)$ for $\zeta \ll1$, is the best among the bounds of perceptron-like maximum margin algorithms. Typically, algorithms which require at least an approximate knowledge of the value of $\gamma_{\rm d}$ to tune their parameters have bounds $O \left((\zeta^{-1}R^2/\gamma_{\rm d}^2)\ln (\zeta^{-1}\left(R/\gamma_{\rm d}\right)^k)\right)$ with $k=1,2$ while algorithms which do not assume such a knowledge have bounds $O \left(\zeta^{-2}R^2/\gamma_{\rm d}^2 \right)$.
\end{remark}

\begin{remark}
\label{steps}
Suppose we are given $\bar{\gamma}_{\rm d}<{\gamma}_{\rm d}$. It may be expressed as $\bar{\gamma}_{\rm d}=(1-\xi)\gamma_{\rm d}$. Setting $\lambda=(2/(2+\delta))\bar{\gamma}_{\rm d}^2/b$ it holds that $\epsilon=1-(2/(2+\delta))(1-\xi)^2>\delta/(2+\delta)$. Then (\ref{fbound1}) gives ${\gamma^{\prime}_{\rm d}}/{\gamma_{\rm d}}> 1-\xi+\left(\xi^2-\delta(1-\xi)\right)/(2+\delta)$. Thus, for $\delta (1-\xi)\le \xi^2$ a solution with margin $\gamma^{\prime}_{\rm d}>\bar{\gamma}_{\rm d}$ is obtained which provides a better lower bound on $\gamma_{\rm d}$ than the one used as an input. A repeated application of this procedure starting, e.g., with $\bar{\gamma}_{\rm d}=0$, $\xi=1$, $\lambda=0$ gives solutions possessing margin which is any desirable approximation of $\gamma_{\rm d}$ without prior knowledge of its value. An estimate of the quality of the approximation at each stage may be obtained via the after-run lower bound (\ref{fboundafter1}) on ${\gamma^{\prime}_{\rm d}}/{\gamma_{\rm d}}$ which provides an upper bound on $\gamma_{\rm d}$. In practice, only a few repetitions of this procedure are required to obtain a satisfactory approximation of the optimal solution because the margin actually achieved by the algorithm is considerably larger than the one suggested by (\ref{fbound1}).
\end{remark}

\begin{remark}
\label{rm3}
From (\ref{upbound1}) we see that for the algorithm described by (\ref{update}) and (\ref{cond}) with $c^s_t=1-\eta \lambda, \;{\bar c}^s_t=1$ it holds that $\left\|\vec a^s_{t}\right\|^2=\left\|\vec a_{t}\right\|^2(1-\eta \lambda)^{2(t-1)} \le \left(\eta R^2+2(1-\eta \lambda)b\right)/\left(\lambda(2-\eta \lambda)\right)$. Thus, it is confirmed in this context the well-known fact that constant shrinking leads to bounded length of the weight vector. 
\end{remark}

To proceed with our analysis of the margin perceptron with variable shrinking we need some inequalities involving sums of powers of integers which we present in the form of lemmas. Their proofs can be found in the Appendix.

\begin{lemma}
\label{lemma1}
Let $n \ge 0$ be an integer. Then, it holds that
\begin{equation}
\label{lm1}
(n+1)\sum_{k=1}^{t}k^n \le t(t+1)^n \enspace .
\end{equation}
\end{lemma}

\begin{lemma}
\label{lemma2}
Let $n \ge 0$ be an integer. Then, it holds that
\begin{equation}
\label{lm2}
(n+1)\sum_{k=1}^{t}k^n \ge (t+1)^{n+1}-\frac{(n+1)^2}{2n+1}(t+1)^n \enspace .
\end{equation}
\end{lemma}

\begin{lemma}
\label{lemma3}
Let $n \ge 0$ be an integer. Then, it holds that
\begin{equation}
\label{lm3}
(2n+1)t \sum_{k=1}^{t}k^{2n} \le (n+1)^2\left(\sum_{k=1}^{t}k^n \right)^2 \enspace .
\end{equation}
\end{lemma}

Now we are ready to move on with the analysis of the variable shrinking case.
\begin{theorem}
\label{theorem2}
The margin perceptron with variable shrinking converges in a finite number $t_{\rm c}$ of updates satisfying the bound
\begin{equation}
\label{tbound2}
t_{\rm c} \le t_{\rm b} \equiv \frac{(n+1)^2}{2n+1}\left(1+\frac{2b}{\eta R^2} \right)\frac{R^2}{\gamma^2_{\rm d}}.
\end{equation}
Moreover, the zero-threshold solution hyperplane possesses margin $\gamma^{\prime}_{\rm d}$ which is a fraction $f$ of the maximum margin $\gamma_{\rm d}$ obeying the inequality 
\begin{equation}
\label{fbound2}
f \equiv \frac{\gamma^{\prime}_{\rm d}}{\gamma_{\rm d}} >\frac{2n+1}{2n+2} \left(1+\frac{\eta R^2}{2b}\right)^{-1}.
\end{equation}
Finally, an after-run lower bound on $f$ involving the margin $\gamma^{\prime}_{\rm d}$ achieved, the length $\left\| \vec a_{t_{\rm c}} \right\|$ of the solution weight vector $\vec a_{t_{\rm c}}$ and the number $t_{\rm c}$ of updates is given by
\begin{equation}
\label{fboundafter2}
f \ge \eta \sum_{k=1}^{t_{\rm c}}k^n  \frac {\gamma^{\prime}_{\rm d}}{\left\| \vec a_{t_{\rm c}} \right\|}.
\end{equation}
\end{theorem}

\begin{proof}
Taking the inner product of (\ref{update22}) with the optimal direction $\vec u$ and using (\ref{gamma}) we get
\[
\vec u \cdot \vec a_{t+1} -\vec u \cdot \vec a_{t} =\eta (t+1)^n \vec u \cdot \vec y_k \ge \eta (t+1)^n \gamma_{\rm d}
\]
a repeated application of which, taking into account that $\vec a_0= \vec 0$, gives 
\begin{equation}
\label{lbound2}
\left\| \vec a_t \right\| \ge \vec u \cdot \vec a_t \ge \eta \gamma_{\rm d}\sum_{k=1}^{t}k^n  \enspace.
\end{equation}
From (\ref{update22}) and (\ref{cond22}) we obtain
\[
\left\|\vec a_{t+1}\right\|^2 -\left\|\vec a_t \right\|^2=\eta^2 (t+1)^{2n} \left\|\vec y_k \right\|^2 +2\eta (t+1)^n \vec a_t \cdot \vec y_k \le  (\eta^2 R^2+2\eta b )(t+1)^{2n} \enspace.
\]
A repeated application of the above inequality, assuming $\vec a_0= \vec 0$, leads to
\begin{equation}
\label{upbound2}
\left\|\vec a_{t}\right\|^2 \le (\eta^2 R^2+2\eta b )\sum_{k=1}^{t}k^{2n} \enspace.
\end{equation}
Combining (\ref{lbound2}) and (\ref{upbound2}) we obtain
\begin{equation}
\label{tbound22}
\eta^2\gamma_{\rm d}^2\left(\sum_{k=1}^{t}k^n \right)^2 \le \left\|\vec a_{t}\right\|^2 \le (\eta^2 R^2+2\eta b )\sum_{k=1}^{t}k^{2n}
\end{equation}
or
\[
t \le \left( \frac{R^2+2b/\eta}{\gamma_{\rm d}^2}\right)\left(t\sum_{k=1}^{t}k^{2n}\right)\left(\sum_{k=1}^{t}k^n \right)^{-2}
\]
which by virtue of (\ref{lm3}) gives (\ref{tbound2}).

Upon convergence of the algorithm in $t_{\rm c}$ updates condition (\ref{cond22}) is violated by all patterns. Therefore, the margin $\gamma^{\prime}_{\rm d}$ achieved satisfies $\gamma^{\prime}_{\rm d} > {b(t_{\rm c}+1)^n}/{\left\| \vec a_{t_{\rm c}} \right\|}$.  
Thus,
\begin{equation}
\label{frac2}
f^2=\frac{{\gamma^{\prime}_{\rm d}}^2}{\gamma_{\rm d}^2}>
\frac{b^2(t_{\rm c}+1)^{2n}}{\gamma_{\rm d}^2\left\| \vec a_{t_{\rm c}} \right\|^2} \ge \frac{b^2(t_{\rm c}+1)^{2n}}{\gamma_{\rm d}^2(\eta^2 R^2+2\eta b )\sum_{k=1}^{t_{\rm c}}k^{2n}}\ge  \frac{(2n+1)b^2}{\gamma_{\rm d}^2(\eta R^2+2b )\eta t_{\rm c}} \enspace.
\end{equation}
Here we replaced $\left\|\vec a_{t_{\rm c}}\right\|^2$ with its upper bound $(\eta^2 R^2+2\eta b )\sum_{k=1}^{t_{\rm c}}k^{2n}$ from (\ref{upbound2}) and $\sum_{k=1}^{t_{\rm c}}k^{2n}$ with its upper bound $t_{\rm c}(t_{\rm c}+1)^{2n}/(2n+1)$ from (\ref{lm1}). Overapproximating $t_{\rm c}$ by $t_{\rm b}$ in (\ref{frac2}) and substituting the value of the latter from (\ref{tbound2}) we get
\[
f^2 > \frac{(2n+1)b^2}{\gamma_{\rm d}^2(\eta R^2+2b )\eta t_{\rm b}}=\left(\frac{(2n+1)}{(n+1)}\frac{b}{(\eta R^2+2b)}\right)^2 \nonumber
\]
from where by taking the square root we obtain (\ref{fbound2}).

Finally, (\ref{fboundafter2}) is readily obtained if in the ratio ${\gamma^{\prime}_{\rm d}}/{\gamma_{\rm d}}$ we employ the upper bound on $\gamma_{\rm d}$ derivable from (\ref{lbound2}). \qed
\end{proof}

\begin{remark}
\label{rm4}
Let us define $\delta \equiv {\eta R^2} /{b}$ and $\epsilon \equiv (n+1)^{-1}$. Then, (\ref{tbound2}) and (\ref{fbound2}) become
\[
t_{\rm c} \le \frac{1}{\epsilon\delta}\left(\frac{1+{\delta}/{2}}{1-{\epsilon}/{2}}\right)\frac{R^2}{\gamma^2_{\rm d}}
\]
and
\[
f>\frac{1-{\epsilon}/{2}}{1+{\delta}/{2}} \enspace,
\]
respectively. The perceptron with margin corresponds to $n=0$ or $\epsilon=1$. If we choose $\delta \ll 1$ while keeping $\epsilon$ (i.e., $n$) fixed the upper bound on the number of updates becomes $O \left(\delta^{-1}R^2/\gamma_{\rm d}^2 \right)$ and the before-run lower bound on $f$ approaches as $\delta \to 0$ the value $1-\epsilon/2$. Then, by allowing $\epsilon \to 0$ (i.e., $n \to \infty$) maximum margin solutions are approximated. If we set, instead, $\delta=\epsilon \ll 1$ then $f>1-\epsilon$ and the algorithm achieves ``accuracy'' $\epsilon$ in at most $\epsilon^{-2}R^2/\gamma_{\rm d}^2+O \left(\epsilon^{-1}R^2/\gamma_{\rm d}^2 \right)$ updates. This is among the best bounds of perceptron-like approximate maximum margin classifiers which do not assume knowledge of the value of $\gamma_{\rm d}$ in any way. For comparison, ALMA's bound is $\simeq 8\epsilon^{-2}R^2/\gamma_{\rm d}^2$.
\end{remark}

\begin{remark}
\label{rm5}
Given that $f^2 \le 1$ (\ref{frac2}) leads to a lower bound on the number $t_{\rm c}$ of updates required for convergence of the margin perceptron with variable shrinking which in terms of the parameters $\delta$ and $\epsilon$ reads
\[
t_{\rm c} > \frac{1}{\epsilon\delta}\left(\frac{1-{\epsilon}/{2}}{1+{\delta}/{2}}\right)\frac{R^2}{\gamma^2_{\rm d}} \enspace.
\]
As $\delta, \epsilon \to 0$ the ratio of the above lower bound to the upper bound tends to $1$ and the algorithm approaches the optimal solution in $\simeq(\epsilon \delta)^{-1}{R^2}/{\gamma^2_{\rm d}}$ updates.
\end{remark}

\begin{remark}
\label{rm6}
Theorems \ref{theorem1} and \ref{theorem2} hold also for the algorithms described by (\ref{update}) and (\ref{cond}) as appropriate provided, of course, that $\left\| \vec a_{t_{\rm c}} \right\|$ is replaced in (\ref{fboundafter1}) and (\ref{fboundafter2}) with $\left\| \vec a^s_{t_{\rm c}} \right\|$ by making use of the relation connecting these two quantities.
\end{remark}

\begin{remark}
\label{afterrun}
The after-run lower bounds on $f$ given by (\ref{fboundafter1}) and (\ref{fboundafter2}) typically provide estimates of the margin achieved which are much more accurate than the ones obtained from the before-run bounds of (\ref{fbound1}) and (\ref{fbound2}), respectively. Our experience based on such estimates suggests that a satisfactory approximation of the maximum margin solution can be obtained without the need to resort to very small values of the parameter $\epsilon$. In other words, although the theoretically guaranteed before-run fraction of the maximum margin for $\delta \ll 1$ is close to $1-\epsilon/2$ both the estimated after-run fraction and the one actually achieved are larger. This is a generic feature of the perceptron with margin and its generalizations. It turns out that in most cases $\epsilon \simeq 0.2-0.3$ is sufficiently small for the algorithm to obtain for $\delta \ll 1$ solutions possessing $99 \%$ of the maximum margin. Thus, for constant shrinking a very accurate knowledge of the value of $\gamma_{\rm d}$ is not required while for variable shrinking very low values of $n$ are sufficient. 
\end{remark}

\section{Implementation and Experiments}
To reduce the computational cost we adopt a two-member nested sequence of reduced ``active sets" of data points as described in detail in \cite{PT1}. The parameter $\bar c$ which multiplies the threshold of the misclassification condition when this condition is used to select the points of the first-level active set is given the value $\bar c=1.01$. The parameters, instead, determining the number of times the active sets are presented to the algorithm are set to the values $N_{\rm ep_1}=N_{\rm ep_2}=5$.

An additional mechanism providing a substantial improvement of the computational efficiency is the one of performing multiple updates \cite{PT2,PT1,PT3} once a data point is presented to the algorithm. It is understood, of course, that a multiple update should be equivalent to a certain number of updates occurring as a result of repeatedly presenting to the algorithm the data point in question. Thus, the maximal multiplicity of such an update will be determined by the requirement that the pattern $\vec y_k$ which satisfies the misclassification condition will just violate it as a result of the multiple update. For constant shrinking a multiple update is 
\[
\vec{a}_{t+\mu}=\vec{a}_{t}+ \frac{1-(1-\eta \lambda)^\mu }{\lambda(1-\eta\lambda)^{t+\mu-1}}\vec{y}_k 
\]
with
\[
\mu \le \left[\frac{1}{\ln (1-\eta \lambda)^{-1}} \ln\left(1+\lambda \frac{b-(1-\eta \lambda)^{t-1} \vec a_t \cdot \vec y_k }{\left\|\vec y_k \right\|^2-\lambda b}\right)\right]+1  \enspace.
\]
Here $[x]$ is the integer part of $x\ge 0$. For variable shrinking, instead, finding the maximal multiplicity of the update involves solving a (n+1)-th degree equation for which there is no general formula unless $n \le 3$. However, this does not pose a serious problem for several reasons.  First of all, as we have already pointed out in Remark \ref{afterrun}, we may reach very good approximations of the maximal margin hyperplane with low values of $n$. In addition, even if we choose a larger $n$ we may obtain satisfactory performance with updates having multiplicity lower than the maximal one. Thus, it suffices to find a lower bound on the relevant root of the (n+1)-th degree equation. Moreover, even when the exact root is available it is often preferable to set an upper bound $\ell_{\rm {up}}$ on the multiplicity of the updates.

The aim of our experiments is to assess the ability of the margin perceptron with constant shrinking (MPCS) and the margin perceptron with variable shrinking (MPVS) to achieve fast convergence to a certain approximation of the optimal solution in the feature space where the patterns are linearly separable. For linearly separable data the feature space is the initial instance space. For inseparable data, instead, a space extended by $m$ dimensions, as many as the instances, is considered where each instance is placed at a distance $\Delta$ from the origin in the corresponding dimension\footnote{$\vec y_k=[ \bar{\vec y}_k, l_k\Delta \delta _{1k}, \dots , l_k \Delta \delta _{mk}]$, where $\delta _{ij}$ is Kronecker's $\delta$ and $\bar{\vec y}_k$ the projection of the $k^{\rm {th}}$ extended instance $\vec y_k$ (multiplied by its label $l_k$) onto the initial instance space. The feature space mapping defined by the extension commutes with a possible augmentation (with parameter $\rho$) in which case $\bar{\vec y}_k=[l_k  \bar{\vec x}_k, l_k \rho]$. Here $\bar{\vec x}_k$ represents the $k^{\rm {th}}$ data point.} \cite{FS}. This extension generates a margin of at least $\Delta/\sqrt{m}$ and its employment relies on the well-known equivalence between the hard margin optimization in the extended space and the soft margin optimization in the initial instance space with objective function $\left\|{\vec w}\right\|^2+ \Delta^{-2} {\sum_i} {{\xi}_i}^2$ involving the weight vector $\vec w$ and the 2-norm of the slacks ${\xi}_i$ \cite{CST}.

In the experiments the augmentation parameter $\rho$ was set to the value $\rho=1$. The values of the parameter $\Delta$ together with the number of instances and attributes of the datasets used are given in Table 1. Further details may be found in \cite{PT3}. The experiments, like the ones of \cite{PT3}, were conducted on a 2.5 GHz Intel Core 2 Duo processor with 3 GB RAM running Windows Vista. Therefore, the runtimes reported here can be directly compared to the ones of \cite{PT3}. Our codes written in C++ were compiled using the g++ compiler under Cygwin. They are available at \url{http://users.auth.gr/costapan}.

\begin{table*}[t]
\caption{Results of experiments with the algorithms MPCS and  MPVS.}
\label{Table1}
\centering
\end{table*}
\begin{figure}[t]
\vspace{-15pt}
\epsfig{file=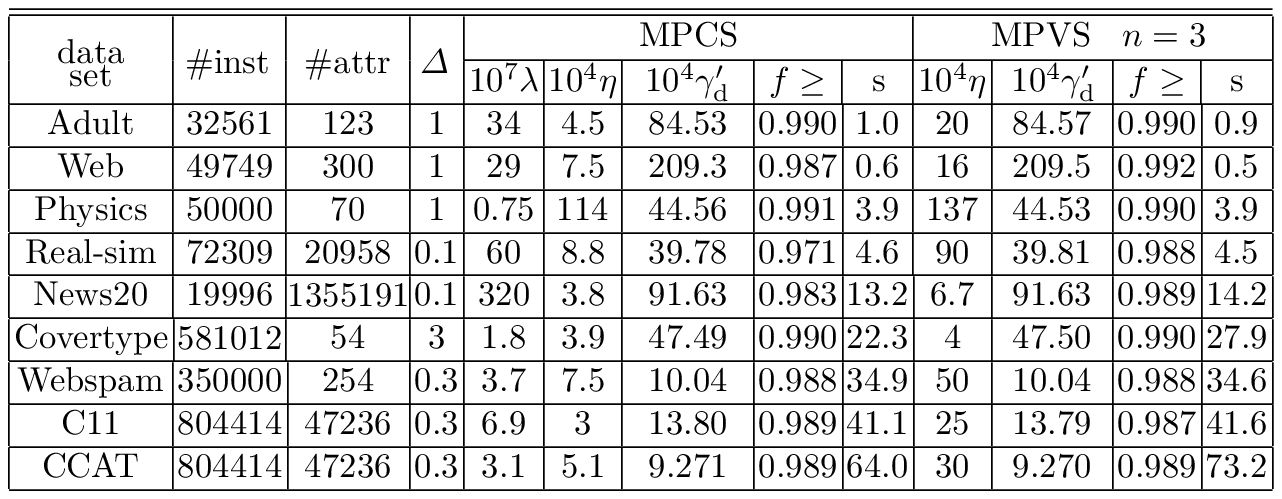, width=1\textwidth}
\vspace{-15pt}
\end{figure}
In the numerical experiments the results of which we report in Table 1 the algorithms MPCS and MPVS were required to obtain solutions possessing $99 \%$ of the maximum margin $\gamma_{\rm d}$. Additionally, we imposed a cut-off value $\ell_{\rm {up}}=1000$ on the multiplicity of the updates. We set $b=R^2$ such that $\delta=\eta R^2/b=\eta$ for both algorithms. For MPCS assuming knowledge of $\gamma_{\rm d}$ we chose $\lambda \simeq 0.75\gamma^2_{\rm d}/b$ such that $\epsilon \simeq 0.25$. In the case of MPVS we set $n=3$ giving $\epsilon=(n+1)^{-1}=0.25$. Thus, for both algorithms the asymptotic value of the theoretically guaranteed fraction of $\gamma_{\rm d}$ that they were able to achieve in the limit $\delta \to 0$ was $1-\epsilon/2 \simeq 0.875$. The lower bound on the fraction $f$ reported is the after-run bound of (\ref{fboundafter1}) and (\ref{fboundafter2}) which turns out in most cases to be $\simeq 0.99$ and certainly much larger than the before-run fraction $\simeq 0.875$ in accordance with our earlier discussion in Remark \ref{afterrun}. The required value of the margin was achieved by sufficiently lowering the value of $\eta$ having knowledge of the target value. However, even if such a knowledge were not available we could have reached our goal guided by the after-run lower bound on $f$. From Table 1 we see that the runtimes (in seconds) of MPCS and MPVS for the same value $\gamma^{\prime}_{\rm d}$ of the margin achieved are comparable. More important, though, is a comparison with the results obtained with other large margin classifiers as reported in \cite{PT3}. We see that MPCS and MPVS are orders of magnitude faster than ROMMA and ${\rm SVM}^{\rm light}$ \cite{Joa}, faster than PDM and of comparable speed or at most about 2 times slower than the linear SVM algorithms DCD \cite{HCL} and MPU \cite {PT2}. We should note, however, that unlike our algorithms linear SVMs are not primal and strictly online.

Finally, we would like to point out that in practice it is possible to set at one stage the parameter $\lambda$ of MPCS without prior knowledge of the value of $\gamma_{\rm d}$. A preliminary run of MPCS with an almost vanishing $\lambda$ provides a lower bound on $\gamma_{\rm d}$ which is the margin $\gamma^{\prime}_{\rm d}$ achieved and an upper bound from the after-run lower bound on $f$. Actually, $\gamma_{\rm d}$ usually lies closer to its upper bound. This information is sufficient to choose $\lambda$ given that the algorithm is not extremely sensitive to this choice provided, of course, that $\lambda$ remains below its maximal allowed value.   

\section{Conclusions}
Motivated by the presence of weight shrinking in most attempts at solving the L1-SVM problem via stochastic gradient descent we introduced this feature into the classical perceptron algorithm with margin. In the case of constant weight decay parameter $\lambda$ and constant learning rate we demonstrated that convergence to solutions with approximately maximum margin requires $\lambda$ to approach a margin-dependent maximal allowed value. Scenarios with variable shrinking strength were also considered and proven not to be subject to such limitations. The theoretical analysis was corroborated by an experimental investigation with massive datasets which involved searching for large margin solutions in an extended feature space, a problem equivalent to the 2-norm soft margin one. As a final conclusion of our study we may say that shrinking of the current weight vector as a first step of the update is able to elevate the margin perceptron to a very effective primal online large margin classifier.

\appendix

\setcounter{equation}{0}
\renewcommand{\theequation}{\Alph{section}.\arabic{equation}}
\section{Proof of Lemma 1}
\begin{proof}
We proceed by induction in the integer $t$. For $t=1$ inequality (\ref{lm1}) reduces to $(n+1) \le 2^n$ which holds since $2^n=(1+1)^n \ge 1+n$. Now let us assume that (\ref{lm1}) holds and prove that $(n+1)\sum_{k=1}^{t+1}k^n \le (t+1)(t+2)^n$ or $(n+1)\left((t+1)^n+\sum_{k=1}^{t}k^n \right) \le (t+1)(t+2)^n$. Given that (\ref{lm1}) holds it suffices to prove that $(n+1)(t+1)^n+t(t+1)^n \le (t+1)(t+2)^n$ or that $(t+2)^n \ge(t+1)^{n-1}(n+1+t)$. Indeed, $(t+2)^n=(t+1)^n\left(1+(t+1)^{-1}\right)^n\ge (t+1)^n\left(1+n(t+1)^{-1}\right)=(t+1)^{n-1}(t+1+n)$. \qed
\end{proof}

\section{Proof of Lemma 2}
\begin{proof}
We proceed by induction in the integer $t$. For $t=1$ inequality (\ref{lm2}) reduces to $(n+1)(2n+1) \ge 2^n\left(1-n(n-2)\right)$ which holds $\forall n\ge 0$. Now let us assume that (\ref{lm2}) holds and prove that $\sum_{k=1}^{t+1}k^n \ge \frac{1}{n+1}(t+2)^{n+1}-\frac{n+1}{2n+1}(t+2)^n$. Using (\ref{lm2}) we have $\sum_{k=1}^{t+1}k^n=(t+1)^n+\sum_{k=1}^{t}k^n \ge (t+1)^n+\frac{1}{n+1}(t+1)^{n+1}-\frac{n+1}{2n+1}(t+1)^n=\frac{1}{n+1}(t+1)^{n+1}+\frac{n}{2n+1}(t+1)^n$. Thus, it suffices to prove that  $F(t) \equiv \frac{n+1}{2n+1}(t+2)^n+\frac{n}{2n+1}(t+1)^n-\frac{1}{n+1}\left((t+2)^{n+1}-(t+1)^{n+1}\right)\ge 0$ or that  $F(t)/t^n \ge 0$.
\begin{comment}
with $F(t)/t^n$ being the function 
\[
\frac{n+1}{2n+1}\left(1+\frac{2}{t}\right)^n+\frac{n}{2n+1}\left(1+\frac{1}{t}\right)^n-\frac{t}{n+1}\left(\left(1+\frac{2}{t}\right)^{n+1}-\left(1+\frac{1}{t}\right)^{n+1}\right)\enspace.
\]
\end{comment}
By virtue of the binomial formula $F(t)/t^n$ admits the expansion
\[
\frac{F(t)}{t^n}=\sum_{l=0}^n\frac{n!}{l!(n-l)!}\left(\frac{(n+1)2^l+n}{2n+1}-\frac{2^{l+1}-1}{l+1}\right)t^{-l}
\enspace.
\]
Given that $((n+1)2^l+n)(l+1)-(2^{l+1}-1)(2n+1)=((l-3)2^l+l+3)n+(l-1)2^l+1 \ge 0\;\;\; \forall l \ge 0$ the terms in the above expansion are all non-negative implying $F(t)/t^n \ge 0$. \qed
\end{proof}

\section{Proof of Lemma 3}
\begin{proof}
We proceed by induction in the integer $t$. For $t=1$ inequality (\ref{lm3}) reduces to $2n+1 \le (n+1)^2$ which obviously holds $\forall n\ge 0$. Now let us assume that (\ref{lm3}) holds and prove that $(2n+1)(t+1) \sum_{k=1}^{t+1}k^{2n} \le (n+1)^2\left(\sum_{k=1}^{t+1}k^n \right)^2 $. Using (\ref{lm3}) we have $(2n+1)(t+1)\sum_{k=1}^{t+1}k^{2n}=(2n+1)(t+1)\left((t+1)^{2n}+\sum_{k=1}^{t}k^{2n}\right)=(2n+1)(t+1)^{2n+1}+\frac{t+1}{t}(2n+1)t \sum_{k=1}^{t}k^{2n}\le (2n+1)(t+1)^{2n+1}+\frac{t+1}{t}(n+1)^2\left(\sum_{k=1}^{t}k^n \right)^2$. Also $(n+1)^2\left(\sum_{k=1}^{t+1}k^n \right)^2=(n+1)^2\left((t+1)^n+\sum_{k=1}^{t}k^n \right)^2=(n+1)^2(t+1)^{2n}+(n+1)^2\left(\sum_{k=1}^{t}k^n \right)^2+2(n+1)^2(t+1)^n\sum_{k=1}^{t}k^n$. Thus, it suffices to prove that $(2n+1)(t+1)^{2n+1}+\frac{t+1}{t}(n+1)^2\left(\sum_{k=1}^{t}k^n \right)^2 \le (n+1)^2(t+1)^{2n}+(n+1)^2\left(\sum_{k=1}^{t}k^n \right)^2+2(n+1)^2(t+1)^n\sum_{k=1}^{t}k^n$ or, equivalently, that $(n+1)\left(2(n+1)t(t+1)^n-(n+1)\sum_{k=1}^{t}k^n\right)\sum_{k=1}^{t}k^n+(n+1)^2t(t+1)^{2n}-(2n+1)t(t+1)^{2n+1} \ge 0$. Replacing in the above inequality $(n+1)\sum_{k=1}^{t}k^n$ with its upper bound $t(t+1)^n$ from (\ref{lm1}) we end up with the inequality $(n+1)(2n+1)t(t+1)^n\sum_{k=1}^{t}k^n+(n+1)^2t(t+1)^{2n}-(2n+1)t(t+1)^{2n+1} \ge 0$ to prove which, however, is equivalent to (\ref{lm2}). \qed
\end{proof}

\end{document}